\documentclass[11pt,twoside,a4paper]{article}
\usepackage{pgm}

\usepackage{amsthm}

\newtheorem{theorem}{Theorem}
\newtheorem{lemma}{Lemma}

\theoremstyle{definition}

\theoremstyle{remark}

\usepackage{graphicx,natbib,amssymb,datetime,float,MnSymbol}
\usepackage{tikz}
\usetikzlibrary{arrows}

\def\ci{\!\perp\!}
\def\nci{\!\not\perp\!}
\def\ra{\rightarrow}
\def\la{\leftarrow}
\def\bb{\leftfootline\!\!\!\!\!\rightfootline}
\def\bo{\leftfootline\!\!\!\!\!\multimap}
\def\ob{\mathrel{\reflectbox{\ensuremath{\bo}}}}
\def\bn{\leftfootline}
\def\nb{\rightfootline}
\def\oo{\mathrel{\reflectbox{\ensuremath{\multimap}}}\!\!\!\!\!\multimap}
\newcommand{\comments}[1]{}

\title{Learning AMP Chain Graphs under Faithfulness}

\author{Jose M. Pe\~{n}a\\
ADIT, IDA, Link\"oping University, SE-58183 Link\"{o}ping, Sweden\\
jose.m.pena@liu.se
}

\summary{This paper deals with chain graphs under the alternative Andersson-Madigan-Perlman (AMP) interpretation. In particular, we present a constraint based algorithm for learning an AMP chain graph a given probability distribution is faithful to. We also show that the extension of Meek's conjecture to AMP chain graphs does not hold, which compromises the development of efficient and correct score+search learning algorithms under assumptions weaker than faithfulness.
}

\begin{document}

\maketitle

\section{Introduction}\label{sec:introduction}

This paper deals with chain graphs (CGs) under the alternative Andersson-Madigan-Perlman (AMP) interpretation \citep{Anderssonetal.2001}. In particular, we present an algorithm for learning an AMP CG a given probability distribution is faithful to. To our knowledge, we are the first to present such an algorithm. However, it is worth mentioning that, under the classical Lauritzen-Wermuth-Frydenberg (LWF) interpretation of CGs \citep{Lauritzen1996}, such an algorithm already exists \citep{Maetal.2008,Studeny1997}. Moreover, we have recently developed an algorithm for learning LWF CGs under the milder composition property assumption \citep{Pennaetal.2012}.

The AMP and LWF interpretations of CGs are sometimes considered as competing and, thus, their relative merits have been pointed out \citep{Anderssonetal.2001,DrtonandEichler2006,Levitzetal.2001,RoveratoandStudeny2006}. Note, however, that no interpretation subsumes the other: There are many independence models that can be induced by a CG under one interpretation but that cannot be induced by any CG under the other interpretation \citep[Theorem 6]{Anderssonetal.2001}.

The rest of the paper is organized as follows. Section \ref{sec:preliminaries} reviews some concepts. Section \ref{sec:algorithm} presents the algorithm. Section \ref{sec:correctness} proves its correctness. Section \ref{sec:discussion} closes with some discussion.

\section{Preliminaries}\label{sec:preliminaries}

In this section, we review some concepts from probabilistic graphical models that are used later in this paper. All the graphs and probability distributions in this paper are defined over a finite set $V$. All the graphs in this paper are hybrid graphs, i.e. they have (possibly) both directed and undirected edges. The elements of $V$ are not distinguished from singletons. We denote by $|X|$ the cardinality of $X \subseteq V$.

If a graph $G$ contains an undirected (resp. directed) edge between two nodes $V_{1}$ and $V_{2}$, then we write that $V_{1} - V_{2}$ (resp. $V_{1} \ra V_{2}$) is in $G$. The parents of a set of nodes $X$ of $G$ is the set $pa_G(X) = \{V_1 | V_1 \ra V_2$ is in $G$, $V_1 \notin X$ and $V_2 \in X \}$. The neighbors of a set of nodes $X$ of $G$ is the set $ne_G(X) = \{V_1 | V_1 - V_2$ is in $G$, $V_1 \notin X$ and $V_2 \in X \}$. The adjacents of a set of nodes $X$ of $G$ is the set $ad_G(X) = \{V_1 | V_1 \ra V_2$, $V_1 - V_2$ or $V_1 \la V_2$ is in $G$, $V_1 \notin X$ and $V_2 \in X \}$. A route from a node $V_{1}$ to a node $V_{n}$ in $G$ is a sequence of (not necessarily distinct) nodes $V_{1}, \ldots, V_{n}$ such that $V_i \in ad_G(V_{i+1})$ for all $1 \leq i < n$. The length of a route is the number of (not necessarily distinct) edges in the route, e.g. the length of the route $V_{1}, \ldots, V_{n}$ is $n-1$. A route is called a cycle if $V_n=V_1$. A route is called descending if $V_{i} \in pa_G(V_{i+1}) \cup ne_G(V_{i+1})$ for all $1 \leq i < n$. The descendants of a set of nodes $X$ of $G$ is the set $de_G(X) = \{V_n |$ there is a descending route from $V_1$ to $V_n$ in $G$, $V_1 \in X$ and $V_n \notin X \}$. A cycle is called a semidirected cycle if it is descending and $V_{i} \ra V_{i+1}$ is in $G$ for some $1 \leq i < n$. A chain graph (CG) is a hybrid graph with no semidirected cycles. A set of nodes of a CG is connected if there exists a route in the CG between every pair of nodes in the set st all the edges in the route are undirected. A connectivity component of a CG is a connected set that is maximal wrt set inclusion. The connectivity component a node $A$ of a CG $G$ belongs to is denoted as $co_G(A)$. The subgraph of $G$ induced by a set of its nodes $X$ is the graph over $X$ that has all and only the edges in $G$ whose both ends are in $X$. An immorality is an induced subgraph of the form $A \ra B \la C$. A flag is an induced subgraph of the form $A \ra B - C$. If $G$ has an induced subgraph of the form $A \ra B \la C$ or $A \ra B - C$, then we say that the triplex $(\{A,C\},B)$ is in $G$. Two CGs are triplex equivalent iff they have the same adjacencies and the same triplexes.

A node $B$ in a route $\rho$ is called a head-no-tail node in $\rho$ if $A \ra B \la C$, $A \ra B - C$, or $A - B \la C$ is a subroute of $\rho$ (note that maybe $A=C$ in the first case). Let $X$, $Y$ and $Z$ denote three disjoint subsets of $V$. A route $\rho$ in a CG $G$ is said to be $Z$-open when (i) every head-no-tail node in $\rho$ is in $Z$, and (ii) every other node in $\rho$ is not in $Z$. When there is no route in $G$ between a node in $X$ and a node in $Y$ that is $Z$-open, we say that $X$ is separated from $Y$ given $Z$ in $G$ and denote it as $X \ci_G Y | Z$. We denote by $X \nci_G Y | Z$ that $X \ci_G Y | Z$ does not hold. Likewise, we denote by $X \ci_p Y | Z$ (resp. $X \nci_p$ $Y | Z$) that $X$ is independent (resp. dependent) of $Y$ given $Z$ in a probability distribution $p$. The independence model induced by $G$, denoted as $I(G)$, is the set of separation statements $X \ci_G$ $Y | Z$. We say that $p$ is Markovian with respect to $G$ when $X \ci_p Y | Z$ if $X \ci_G Y | Z$ for all $X$, $Y$ and $Z$ disjoint subsets of $V$. We say that $p$ is faithful to $G$ when $X \ci_p Y | Z$ iff $X \ci_G Y | Z$ for all $X$, $Y$ and $Z$ disjoint subsets of $V$. If two CGs $G$ and $H$ are triplex equivalent, then $I(G)=I(H)$.\footnote{To see it, note that there are Gaussian distributions $p$ and $q$ that are faithful to $G$ and $H$, respectively \citep[Theorem 6.1]{Levitzetal.2001}. Moreover, $p$ and $q$ are Markovian wrt $H$ and $G$, respectively, by \citet[Theorem 5]{Anderssonetal.2001} and \citet[Theorem 4.1]{Levitzetal.2001}.}

Let $X$, $Y$, $Z$ and $W$ denote four disjoint subsets of $V$. Any probability distribution $p$ satisfies the following properties: Symmetry $X \ci_p Y | Z \Rightarrow Y \ci_p X | Z$, decomposition $X \ci_p$ $Y \cup W | Z \Rightarrow X \ci_p Y | Z$, weak union $X \ci_p$ $Y \cup W | Z \Rightarrow X \ci_pY | Z \cup W$, and contraction $X \ci_p Y | Z \cup W \land X \ci_p W | Z \Rightarrow X \ci_p Y \cup W | Z$. Moreover, if $p$ is faithful to a CG, then it also satisfies the following properties: Intersection $X \ci_p Y | Z \cup W \land X \ci_p W | Z \cup Y \Rightarrow X \ci_p Y \cup W | Z$, and composition $X \ci_p Y | Z \land X \ci_p W | Z \Rightarrow X \ci_p$ $Y \cup W | Z$.\footnote{To see it, note that there is a Gaussian distribution that is faithful to $G$ \citep[Theorem 6.1]{Levitzetal.2001}. Moreover, every Gaussian distribution satisfies the intersection and composition properties \citep[Proposition 2.1 and Corollary 2.4]{Studeny2005}.}

\section{The Algorithm}\label{sec:algorithm}

Our algorithm, which can be seen in Table \ref{tab:algorithm}, resembles the well-known PC algorithm \citep{Meek1995,Spirtesetal.1993}. It consists of two phases: The first phase (lines 1-9) aims at learning adjacencies, whereas the second phase (lines 10-11) aims at directing some of the adjacencies learnt. Specifically, the first phase declares that two nodes are adjacent iff they are not separated by any set of nodes. Note that the algorithm does not test every possible separator (see line 6). Note also that the separators tested are tested in increasing order of size (see lines 2, 6 and 9). The second phase consists of two steps. In the first step, the ends of some of the edges learnt in the first phase are blocked according to the rules R1-R4 in Table \ref{tab:rules}. A block is represented by a perpendicular line such as in $\bn$ or $\bb$, and it means that the edge cannot be directed in that direction. In the second step, the edges with exactly one unblocked end get directed in the direction of the unblocked end. The rules R1-R4 work as follows: If the conditions in the antecedent of a rule are satisfied, then the modifications in the consequent of the rule are applied. Note that the ends of some of the edges in the rules are labeled with a circle such as in $\bo$ or $\oo$. The circle represents an unspecified end, i.e. a block or nothing. The modifications in the consequents of the rules consist in adding some blocks. Note that only the blocks that appear in the consequents are added, i.e. the circled ends do not get modified. The conditions in the antecedents of R1, R2 and R4 consist of an induced subgraph of $H$ and the fact that some of its nodes are or are not in some separators found in line 7. The condition in the antecedent of R3 is slightly different as it only says that there is a cycle in $H$ whose edges have certain blocks, i.e. it says nothing about the subgraph induced by the nodes in the cycle or whether these nodes belong to some separators or not. Note that, when considering the application of R3, one does not need to consider intersecting cycles, i.e. cycles containing repeated nodes other than the initial and final ones.

\begin{table}[t]
\caption{The algorithm.}\label{tab:algorithm}
\centering
\scalebox{0.8}{
\begin{tabular}{rl}
\\
\hline
\\
& Input: A probability distribution $p$ that is faithful\\
& to an unknown CG $G$.\\
& Output: A CG $H$ that is triplex equivalent to $G$.\\
\\
1 & Let $H$ denote the complete undirected graph\\
2 & Set $l=0$\\
3 & Repeat while possible\\
4 & \hspace{0.2cm} Repeat while possible\\
5 & \hspace{0.5cm} Select any ordered pair of nodes $A$ and $B$ in $H$\\
& \hspace{0.5cm} st $A \in ad_H(B)$ and\\
& \hspace{0.5cm} $|[ad_H(A) \cup ad_H(ad_H(A)) ] \setminus B| \geq l$\\
6 & \hspace{0.5cm} If there exists some\\
& \hspace{0.5cm} $S \subseteq [ad_H(A) \cup ad_H(ad_H(A)) ] \setminus B$ st $|S|=l$ and\\
& \hspace{0.5cm} $A \ci_p B | S$ then\\
7 & \hspace{0.8cm} Set $S_{AB}=S_{BA}=S$\\
8 & \hspace{0.8cm} Remove the edge $A - B$ from $H$\\
9 & \hspace{0.2cm} Set $l=l+1$\\
10 & Apply the rules R1-R4 to $H$ while possible\\
11 & Replace every edge $\bn$ ($\bb$) in $H$ with $\ra$ ($-$)\\
\\
\hline
\\
\end{tabular}}
\end{table}

\begin{table}[t]
\caption{The rules R1-R4.}\label{tab:rules}
\centering
\scalebox{0.75}{
\begin{tabular}{cccc}
\\
\hline
\\
R1:&
\begin{tabular}{c}
\begin{tikzpicture}[inner sep=1mm]
\node at (0,0) (A) {$A$};
\node at (1.5,0) (B) {$B$};
\node at (3,0) (C) {$C$};
\path[o-o] (A) edge (B);
\path[o-o] (B) edge (C);
\end{tikzpicture} 
\end{tabular}
& $\Rightarrow$ &
\begin{tabular}{c}
\begin{tikzpicture}[inner sep=1mm]
\node at (0,0) (A) {$A$};
\node at (1.5,0) (B) {$B$};
\node at (3,0) (C) {$C$};
\path[|-o] (A) edge (B);
\path[o-|] (B) edge (C);
\end{tikzpicture}
\end{tabular}\\
& $\land$ $B \notin S_{AC}$\\
\\
\hline
\\
R2:&
\begin{tabular}{c}
\begin{tikzpicture}[inner sep=1mm]
\node at (0,0) (A) {$A$};
\node at (1.5,0) (B) {$B$};
\node at (3,0) (C) {$C$};
\path[|-o] (A) edge (B);
\path[o-o] (B) edge (C);
\end{tikzpicture}
\end{tabular}
& $\Rightarrow$ &
\begin{tabular}{c}
\begin{tikzpicture}[inner sep=1mm]
\node at (0,0) (A) {$A$};
\node at (1.5,0) (B) {$B$};
\node at (3,0) (C) {$C$};
\path[|-o] (A) edge (B);
\path[|-o] (B) edge (C);
\end{tikzpicture}
\end{tabular}\\
& $\land$ $B \in S_{AC}$\\
\\
\hline
\\
R3:&
\begin{tabular}{c}
\begin{tikzpicture}[inner sep=1mm]
\node at (0,0) (A) {$A$};
\node at (1.5,0) (B) {$\ldots$};
\node at (3,0) (C) {$B$};
\path[|-o] (A) edge (B);
\path[|-o] (B) edge (C);
\path[o-o] (A) edge [bend left] (C);
\end{tikzpicture}
\end{tabular}
& $\Rightarrow$ &
\begin{tabular}{c}
\begin{tikzpicture}[inner sep=1mm]
\node at (0,0) (A) {$A$};
\node at (1.5,0) (B) {$\ldots$};
\node at (3,0) (C) {$B$};
\path[|-o] (A) edge (B);
\path[|-o] (B) edge (C);
\path[|-o] (A) edge [bend left] (C);
\end{tikzpicture}
\end{tabular}\\
\\
\hline
\\
R4:&
\begin{tabular}{c}
\begin{tikzpicture}[inner sep=1mm]
\node at (0,0) (A) {$A$};
\node at (2,0) (B) {$B$};
\node at (1,1) (C) {$C$};
\node at (1,-1) (D) {$D$};
\path[o-o] (A) edge (B);
\path[o-o] (A) edge (C);
\path[o-o] (A) edge (D);
\path[|-o] (C) edge (B);
\path[|-o] (D) edge (B);
\end{tikzpicture}
\end{tabular}
& $\Rightarrow$ &
\begin{tabular}{c}
\begin{tikzpicture}[inner sep=1mm]
\node at (0,0) (A) {$A$};
\node at (2,0) (B) {$B$};
\node at (1,1) (C) {$C$};
\node at (1,-1) (D) {$D$};
\path[|-o] (A) edge (B);
\path[o-o] (A) edge (C);
\path[o-o] (A) edge (D);
\path[|-o] (C) edge (B);
\path[|-o] (D) edge (B);
\end{tikzpicture}
\end{tabular}\\
& $\land$ $A \in S_{CD} \land B \notin S_{CD}$\\
\\
\hline
\\
\end{tabular}}
\end{table}

\section{Correctness of the Algorithm}\label{sec:correctness}

In this section, we prove that our algorithm is correct, i.e. it returns a CG the given probability distribution is faithful to. We start proving a result for any probability distribution that satisfies the intersection and composition properties. Recall that any probability distribution that is faithful to a CG satisfies these properties and, thus, the following result applies to it.

\begin{lemma}\label{lem:conditions}
Let $p$ denote a probability distribution that satisfies the intersection and composition properties. Then, $p$ is Markovian wrt a CG $G$ iff $p$ satisfies the following conditions:
\begin{itemize}
\setlength{\itemsep}{-0.1cm}
\item[C1:] $A \ci_p co_G(A) \setminus A \setminus ne_G(A) | pa_G(A \cup ne_G(A)) \cup ne_G(A)$ for all $A \in V$, and

\item[C2:] $A \ci_p V \setminus A \setminus de_G(A) \setminus pa_G(A) | pa_G(A)$ for all $A \in V$.

\end{itemize}

\end{lemma}

\begin{proof}
It follows from \citet[Theorem 3]{Anderssonetal.2001} and \citet[Theorem 4.1]{Levitzetal.2001} that $p$ is Markovian wrt $G$ iff $p$ satisfies the following conditions:
\begin{itemize}
\setlength{\itemsep}{-0.1cm}
\item[L1:] $A \ci_p co_G(A) \setminus A \setminus ne_G(A) | [ V \setminus co_G(A) \setminus de_G(co_G(A)) ] \cup ne_G(A)$ for all $A \in V$, and

\item[L2:] $A \ci_p V \setminus co_G(A) \setminus de_G(co_G(A)) \setminus pa_G(A) | pa_G(A)$ for all $A \in V$.

\end{itemize}

Clearly, C2 holds iff L2 holds because $de_G(A) = [ co_G(A) \cup de_G(co_G(A)) ] \setminus A$. We prove below that if L2 holds, then C1 holds iff L1 holds. We first prove the if part.

\begin{itemize}
\setlength{\itemsep}{-0.1cm}
\item[1.] $B \ci_p V \setminus co_G(B) \setminus de_G(co_G(B)) \setminus pa_G(B) | pa_G(B)$ for all $B \in A \cup ne_G(A)$ by L2.

\item[2.] $B \ci_p V \setminus co_G(B) \setminus de_G(co_G(B)) \setminus pa_G(A \cup ne_G(A)) | pa_G(A \cup ne_G(A))$ for all $B \in A \cup ne_G(A)$ by weak union on 1.

\item[3.] $A \cup ne_G(A) \ci_p V \setminus co_G(A) \setminus de_G(co_G(A)) \setminus pa_G(A \cup ne_G(A)) | pa_G(A \cup ne_G(A))$ by repeated application of symmetry and composition on 2.

\item[4.] $A \ci_p V \setminus co_G(A) \setminus de_G(co_G(A)) \setminus pa_G(A \cup ne_G(A)) | pa_G(A \cup ne_G(A)) \cup ne_G(A)$ by symmetry and weak union on 3.

\item[5.] $A \ci_p co_G(A) \setminus A \setminus ne_G(A) | [ V \setminus co_G(A) \setminus de_G(co_G(A)) ] \cup ne_G(A)$ by L1.

\item[6.] $A \ci_p [ co_G(A) \setminus A \setminus ne_G(A) ] [ V \setminus co_G(A) \setminus de_G(co_G(A)) \setminus pa_G(A \cup ne_G(A)) ] | pa_G(A \cup ne_G(A)) \cup ne_G(A)$ by contraction on 4 and 5.

\item[7.] $A \ci_p co_G(A) \setminus A \setminus ne_G(A) | pa_G(A \cup ne_G(A)) \cup ne_G(A)$ by decomposition on 6.

\end{itemize}

We now prove the only if part.

\begin{itemize}
\setlength{\itemsep}{-0.1cm}
\item[8.] $A \ci_p co_G(A) \setminus A \setminus ne_G(A) | pa_G(A \cup ne_G(A)) \cup ne_G(A)$ by C1.

\item[9.] $A \ci_p [ V \setminus co_G(A) \setminus de_G(co_G(A)) \setminus pa_G(A \cup ne_G(A)) ] [ co_G(A) \setminus A \setminus ne_G(A) ] | pa_G(A \cup ne_G(A)) \cup ne_G(A)$ by composition on 4 and 8.

\item[10.] $A \ci_p co_G(A) \setminus A \setminus ne_G(A) | [ V \setminus co_G(A) \setminus de_G(co_G(A)) ] \cup ne_G(A)$ by weak union on 9.

\end{itemize}

\end{proof}

\begin{lemma}\label{lem:adjacencies}
After line 9, $G$ and $H$ have the same adjacencies.
\end{lemma}

\begin{proof}
Consider any pair of nodes $A$ and $B$ in $G$. If $A \in ad_G(B)$, then $A \nci_p B | S$ for all $S \subseteq V \setminus [ A \cup B ]$ by the faithfulness assumption. Consequently, $A \in ad_H(B)$ at all times. On the other hand, if $A \notin ad_G(B)$, then consider the following cases. 

\begin{description}
\setlength{\itemsep}{-0.1cm}
\item[Case 1] Assume that $co_G(A)=co_G(B)$. Then, $A \ci_p co_G(A) \setminus A \setminus ne_G(A) | pa_G(A \cup ne_G(A)) \cup ne_G(A)$ by C1 in Lemma \ref{lem:conditions} and, thus, $A \ci_p B | pa_G(A \cup ne_G(A)) \cup ne_G(A)$ by decomposition and $B \notin ne_G(A)$, which follows from $A \notin ad_G(B)$. Note that, as shown above, $pa_G(A \cup ne_G(A)) \cup ne_G(A) \subseteq [ ad_H(A) \cup ad_H(ad_H(A)) ] \setminus B$ at all times.

\item[Case 2] Assume that $co_G(A) \neq co_G(B)$. Then, $A \notin de_G(B)$ or $B \notin de_G(A)$ because $G$ has no semidirected cycle. Assume without loss of generality that $B \notin de_G(A)$. Then, $A \ci_p$ $V \setminus A \setminus de_G(A) \setminus pa_G(A) | pa_G(A)$ by C2 in Lemma \ref{lem:conditions} and, thus, $A \ci_p B | pa_G(A)$ by decomposition, $B \notin de_G(A)$, and $B \notin pa_G(A)$ which follows from $A \notin ad_G(B)$. Note that, as shown above, $pa_G(A)\subseteq ad_H(A) \setminus B$ at all times.

\end{description}

Therefore, in either case, there will exist some $S$ in line 6 such that $A \ci_p B | S$ and, thus, the edge $A - B$ will be removed from $H$ in line 7. Consequently, $A \notin ad_H(B)$ after line 9. 
\end{proof}

The next lemma proves that the rules R1-R4 are sound, i.e. if the antecedent holds in $G$, then so does the consequent.

\begin{lemma}\label{lem:soundness}
The rules R1-R4 are sound.
\end{lemma}

\begin{proof}
According to the antecedent of R1, $G$ has a triplex $(\{A,C\},B)$. Then, $G$ has an induced subgraph of the form $A \ra B \la C$, $A \ra B - C$ or $A - B \la C$. In either case, the consequent of R1 holds.

According to the antecedent of R2, (i) $G$ does not have a triplex $(\{A,C\},B)$, (ii) $A \ra B$ or $A - B$ is in $G$, (iii) $B \in ad_G(C)$, and (iv) $A \notin ad_G(C)$. Then, $B \ra C$ or $B - C$ is in $G$. In either case, the consequent of R2 holds.

According to the antecedent of R3, (i) $G$ has a descending route from $A$ to $B$, and (ii) $A \in ad_G(B)$. Then, $A \ra B$ or $A - B$ is in $G$, because $G$ has no semidirected cycle. In either case, the consequent of R3 holds. 

To appreciate the soundness of R4, assume to the contrary that $A \la B$ is in $G$. Then, it follows from applying R3 to the antecedent of R4 that $G$ has an induced subgraph that is consistent with

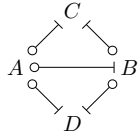
\begin{table}[H]
\centering
\scalebox{0.75}{
\begin{tikzpicture}[inner sep=1mm]
\node at (0,0) (A) {$A$};
\node at (2,0) (B) {$B$};
\node at (1,1) (C) {$C$};
\node at (1,-1) (D) {$D$};
\path[o-|] (A) edge (B);
\path[o-|] (A) edge (C);
\path[o-|] (A) edge (D);
\path[|-o] (C) edge (B);
\path[|-o] (D) edge (B);
\end{tikzpicture}.}
\end{table}

Moreover, recall from the antecedent of R4 that $A \in S_{CD}$, which implies that $G$ does not have a triplex $(\{C,D\},A)$, which implies that $A - C$ and $A - D$ are in $G$. Thus, $G$ has an induced subgraph that is consistent with

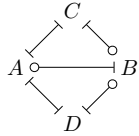
\begin{table}[H]
\centering
\scalebox{0.75}{
\begin{tikzpicture}[inner sep=1mm]
\node at (0,0) (A) {$A$};
\node at (2,0) (B) {$B$};
\node at (1,1) (C) {$C$};
\node at (1,-1) (D) {$D$};
\path[o-|] (A) edge (B);
\path[|-|] (A) edge (C);
\path[|-|] (A) edge (D);
\path[|-o] (C) edge (B);
\path[|-o] (D) edge (B);
\end{tikzpicture}.}
\end{table}

However, recall from the antecedent of R4 that $B \notin S_{CD}$, which implies that $G$ has a triplex $(\{C,D\},B)$, which implies that $G$ has a semidirected cycle, which is a contradiction. Therefore, $A - B$ or $A \ra B$ is in $G$. In either case, the consequent of R4 holds. 
\end{proof}

\begin{lemma}\label{lem:triplexes}
After line 11, $G$ and $H$ have the same triplexes. Moreover, $H$ has all the immoralities in $G$.
\end{lemma}

\begin{proof}
We first prove that any triplex in $H$ is in $G$. Assume to the contrary that $H$ has a triplex $(\{A,C\},B)$ that is not in $G$. This is possible iff, when line 11 is executed, $H$ has an induced subgraph of one of the following forms:

\begin{table}[H]
\centering
\scalebox{0.75}{
\begin{tabular}{ccc}
\begin{tikzpicture}[inner sep=1mm]
\node at (0,0) (A) {$A$};
\node at (1,0) (B) {$B$};
\node at (2,0) (C) {$C$};
\path[-] (A) edge (B);
\path[-|] (B) edge (C);
\end{tikzpicture}
&
\begin{tikzpicture}[inner sep=1mm]
\node at (0,0) (A) {$A$};
\node at (1,0) (B) {$B$};
\node at (2,0) (C) {$C$};
\path[|-] (A) edge (B);
\path[-|] (B) edge (C);
\end{tikzpicture}
&
\begin{tikzpicture}[inner sep=1mm]
\node at (0,0) (A) {$A$};
\node at (1,0) (B) {$B$};
\node at (2,0) (C) {$C$};
\path[|-|] (A) edge (B);
\path[-|] (B) edge (C);
\end{tikzpicture}
\end{tabular}}
\end{table}

where $B \in S_{AC}$ by Lemma \ref{lem:adjacencies}. The first and second forms are impossible because, otherwise, $A \ob B$ would be in $H$ by R2. The third form is impossible because, otherwise, $B \bb C$ would be in $H$ by R2.

We now prove that any triplex $(\{A,C\},B)$ in $G$ is in $H$. Let the triplex be of the form $A \ra B \la C$. Then, when line 11 is executed, $A \bo B \ob C$ is in $H$ by R1, and neither $A \bb B$ nor $B \bb C$ is in $H$ by Lemmas \ref{lem:adjacencies} and \ref{lem:soundness}. Then, the triplex is in $H$. Note that the triplex is an immorality in both $G$ and $H$. Likewise, let the triplex be of the form $A \ra B \bb C$. Then, when line 11 is executed, $A \bo B \ob C$ is in $H$ by R1, and $A \bb B$ is not in $H$ by Lemmas \ref{lem:adjacencies} and \ref{lem:soundness}. Then, the triplex is in $H$. Note that the triplex is a flag in $G$ but it may be an immorality in $H$.
\end{proof}

\begin{lemma}\label{lem:noundirectedcycle}
After line 10, $H$ does not have any induced subgraph of the form 
\begin{tabular}{c}
\scalebox{0.75}{
\begin{tikzpicture}[inner sep=1mm]
\node at (0,0) (A) {$A$};
\node at (1,0) (B) {$B$};
\node at (2,0) (C) {$C$};
\path[|-o] (A) edge (B);
\path[-] (B) edge (C);
\path[-] (A) edge [bend left] (C);
\end{tikzpicture}.}
\end{tabular}
\end{lemma}

\begin{proof}
Assume to the contrary that the lemma does not hold. Consider the following cases.

\begin{description}
\setlength{\itemsep}{-0.1cm}
\item[Case 1] Assume that $A \bo B$ is in $H$ due to R1. That is, $H$ has an induced subgraph of one of the following forms: 

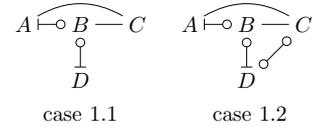
\begin{table}[H]
\centering
\scalebox{0.75}{
\begin{tabular}{cc}
\begin{tikzpicture}[inner sep=1mm]
\node at (0,0) (A) {$A$};
\node at (1,0) (B) {$B$};
\node at (2,0) (C) {$C$};
\node at (1,-1) (D) {$D$};
\path[|-o] (A) edge (B);
\path[-] (B) edge (C);
\path[-] (A) edge [bend left] (C);
\path[|-o] (D) edge (B);
\end{tikzpicture}
&
\begin{tikzpicture}[inner sep=1mm]
\node at (0,0) (A) {$A$};
\node at (1,0) (B) {$B$};
\node at (2,0) (C) {$C$};
\node at (1,-1) (D) {$D$};
\path[|-o] (A) edge (B);
\path[-] (B) edge (C);
\path[-] (A) edge [bend left] (C);
\path[|-o] (D) edge (B);
\path[o-o] (D) edge (C);
\end{tikzpicture}.\\
case 1.1&case 1.2
\end{tabular}}
\end{table}

\begin{description}
\setlength{\itemsep}{-0.1cm}
\item[Case 1.1] If $B \notin S_{CD}$ then $B \nb C$ is in $H$ by R1, else $B \bn C$ is in $H$ by R2. Either case is a contradiction.

\item[Case 1.2] If $C \notin S_{AD}$ then $A \bn C$ is in $H$ by R1, else $B \nb C$ is in $H$ by R4. Either case is a contradiction.

\end{description}

\item[Case 2] Assume that $A \bo B$ is in $H$ due to R2. That is, $H$ has an induced subgraph of one of the following forms: 

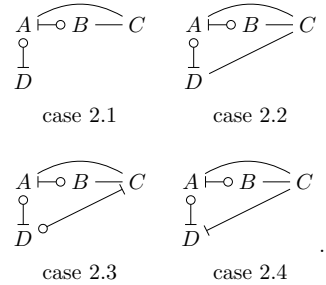
\begin{table}[H]
\centering
\scalebox{0.75}{
\begin{tabular}{cc}
\begin{tikzpicture}[inner sep=1mm]
\node at (0,0) (A) {$A$};
\node at (1,0) (B) {$B$};
\node at (2,0) (C) {$C$};
\node at (0,-1) (D) {$D$};
\path[|-o] (A) edge (B);
\path[-] (B) edge (C);
\path[-] (A) edge [bend left] (C);
\path[|-o] (D) edge (A);
\end{tikzpicture}
&
\begin{tikzpicture}[inner sep=1mm]
\node at (0,0) (A) {$A$};
\node at (1,0) (B) {$B$};
\node at (2,0) (C) {$C$};
\node at (0,-1) (D) {$D$};
\path[|-o] (A) edge (B);
\path[-] (B) edge (C);
\path[-] (A) edge [bend left] (C);
\path[|-o] (D) edge (A);
\path[-] (D) edge (C);
\end{tikzpicture}\\
case 2.1&case 2.2\\
\\
\begin{tikzpicture}[inner sep=1mm]
\node at (0,0) (A) {$A$};
\node at (1,0) (B) {$B$};
\node at (2,0) (C) {$C$};
\node at (0,-1) (D) {$D$};
\path[|-o] (A) edge (B);
\path[-] (B) edge (C);
\path[-] (A) edge [bend left] (C);
\path[|-o] (D) edge (A);
\path[o-|] (D) edge (C);
\end{tikzpicture}
&
\begin{tikzpicture}[inner sep=1mm]
\node at (0,0) (A) {$A$};
\node at (1,0) (B) {$B$};
\node at (2,0) (C) {$C$};
\node at (0,-1) (D) {$D$};
\path[|-o] (A) edge (B);
\path[-] (B) edge (C);
\path[-] (A) edge [bend left] (C);
\path[|-o] (D) edge (A);
\path[|-] (D) edge (C);
\end{tikzpicture}.\\
case 2.3&case 2.4
\end{tabular}}
\end{table}

\begin{description}
\setlength{\itemsep}{-0.1cm}
\item[Case 2.1] If $A \notin S_{CD}$ then $A \nb C$ is in $H$ by R1, else $A \bn C$ is in $H$ by R2. Either case is a contradiction.

\item[Case 2.2] Restart the proof with $D$ instead of $A$ and $A$ instead of $B$.

\item[Case 2.3] Then, $A \nb C$ is in $H$ by R3, which is a contradiction.

\item[Case 2.4] If $C \notin S_{BD}$ then $B \bn C$ is in $H$ by R1, else $B \nb C$ is in $H$ by R2. Either case is a contradiction.

\end{description}

\item[Case 3] Assume that $A \bo B$ is in $H$ due to R3. That is, $H$ has an induced subgraph of one of the following forms: 

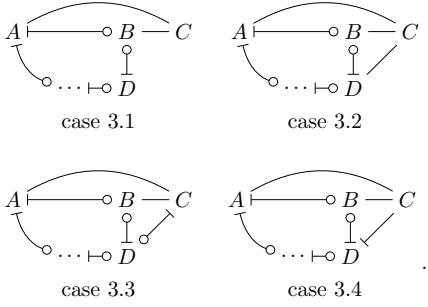
\begin{table}[H]
\centering
\scalebox{0.75}{
\begin{tabular}{cccc}
\begin{tikzpicture}[inner sep=1mm]
\node at (-1,0) (A) {$A$};
\node at (1,0) (B) {$B$};
\node at (2,0) (C) {$C$};
\node at (1,-1) (D) {$D$};
\node at (0,-1) (E) {$\ldots$};
\path[|-o] (A) edge (B);
\path[-] (B) edge (C);
\path[-] (A) edge [bend left] (C);
\path[|-o] (A) edge [bend right] (E);
\path[|-o] (E) edge (D);
\path[|-o] (D) edge (B);
\end{tikzpicture}
&
\begin{tikzpicture}[inner sep=1mm]
\node at (-1,0) (A) {$A$};
\node at (1,0) (B) {$B$};
\node at (2,0) (C) {$C$};
\node at (1,-1) (D) {$D$};
\node at (0,-1) (E) {$\ldots$};
\path[|-o] (A) edge (B);
\path[-] (B) edge (C);
\path[-] (A) edge [bend left] (C);
\path[|-o] (A) edge [bend right] (E);
\path[|-o] (E) edge (D);
\path[|-o] (D) edge (B);
\path[-] (D) edge (C);
\end{tikzpicture}\\
case 3.1&case 3.2\\
\\
\begin{tikzpicture}[inner sep=1mm]
\node at (-1,0) (A) {$A$};
\node at (1,0) (B) {$B$};
\node at (2,0) (C) {$C$};
\node at (1,-1) (D) {$D$};
\node at (0,-1) (E) {$\ldots$};
\path[|-o] (A) edge (B);
\path[-] (B) edge (C);
\path[-] (A) edge [bend left] (C);
\path[|-o] (A) edge [bend right] (E);
\path[|-o] (E) edge (D);
\path[|-o] (D) edge (B);
\path[o-|] (D) edge (C);
\end{tikzpicture}
&
\begin{tikzpicture}[inner sep=1mm]
\node at (-1,0) (A) {$A$};
\node at (1,0) (B) {$B$};
\node at (2,0) (C) {$C$};
\node at (1,-1) (D) {$D$};
\node at (0,-1) (E) {$\ldots$};
\path[|-o] (A) edge (B);
\path[-] (B) edge (C);
\path[-] (A) edge [bend left] (C);
\path[|-o] (A) edge [bend right] (E);
\path[|-o] (E) edge (D);
\path[|-o] (D) edge (B);
\path[|-] (D) edge (C);
\end{tikzpicture}.\\
case 3.3&case 3.4
\end{tabular}}
\end{table}

\begin{description}
\setlength{\itemsep}{-0.1cm}
\item[Case 3.1] If $B \notin S_{CD}$ then $B \nb C$ is in $H$ by R1, else $B \bn C$ is in $H$ by R2. Either case is a contradiction.

\item[Case 3.2] Restart the proof with $D$ instead of $A$.

\item[Case 3.3] Then, $B \nb C$ is in $H$ by R3, which is a contradiction.

\item[Case 3.4] Then, $A \bn C$ is in $H$ by R3, which is a contradiction.

\end{description}

\item[Case 4] Assume that $A \bo B$ is in $H$ due to R4. That is, $H$ has an induced subgraph of one of the following forms:

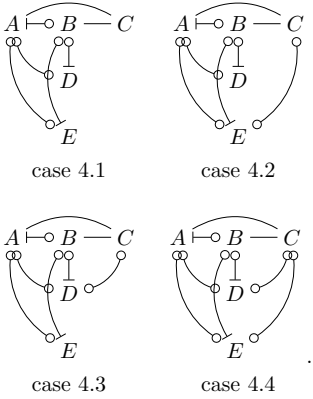
\begin{table}[H]
\centering
\scalebox{0.75}{
\begin{tabular}{ccccc}
\begin{tikzpicture}[inner sep=1mm][inner sep=1mm]
\node at (0,0) (A) {$A$};
\node at (1,0) (B) {$B$};
\node at (2,0) (C) {$C$};
\node at (1,-1) (D) {$D$};
\node at (1,-2) (E) {$E$};
\path[|-o] (A) edge (B);
\path[-] (B) edge (C);
\path[-] (A) edge [bend left] (C);
\path[o-o] (A) edge [bend right] (D);
\path[|-o] (D) edge (B);
\path[o-o] (A) edge [bend right] (E);
\path[|-o] (E) edge [bend left] (B);
\end{tikzpicture}
&
\begin{tikzpicture}[inner sep=1mm][inner sep=1mm]
\node at (0,0) (A) {$A$};
\node at (1,0) (B) {$B$};
\node at (2,0) (C) {$C$};
\node at (1,-1) (D) {$D$};
\node at (1,-2) (E) {$E$};
\path[|-o] (A) edge (B);
\path[-] (B) edge (C);
\path[-] (A) edge [bend left] (C);
\path[o-o] (A) edge [bend right] (D);
\path[|-o] (D) edge (B);
\path[o-o] (A) edge [bend right] (E);
\path[|-o] (E) edge [bend left] (B);
\path[o-o] (E) edge [bend right] (C);
\end{tikzpicture}\\
case 4.1&case 4.2\\
\\
\begin{tikzpicture}[inner sep=1mm][inner sep=1mm]
\node at (0,0) (A) {$A$};
\node at (1,0) (B) {$B$};
\node at (2,0) (C) {$C$};
\node at (1,-1) (D) {$D$};
\node at (1,-2) (E) {$E$};
\path[|-o] (A) edge (B);
\path[-] (B) edge (C);
\path[-] (A) edge [bend left] (C);
\path[o-o] (A) edge [bend right] (D);
\path[|-o] (D) edge (B);
\path[o-o] (A) edge [bend right] (E);
\path[|-o] (E) edge [bend left] (B);
\path[o-o] (D) edge [bend right] (C);
\end{tikzpicture}
&
\begin{tikzpicture}[inner sep=1mm]
\node at (0,0) (A) {$A$};
\node at (1,0) (B) {$B$};
\node at (2,0) (C) {$C$};
\node at (1,-1) (D) {$D$};
\node at (1,-2) (E) {$E$};
\path[|-o] (A) edge (B);
\path[-] (B) edge (C);
\path[-] (A) edge [bend left] (C);
\path[o-o] (A) edge [bend right] (D);
\path[|-o] (D) edge (B);
\path[o-o] (A) edge [bend right] (E);
\path[|-o] (E) edge [bend left] (B);
\path[o-o] (D) edge [bend right] (C);
\path[o-o] (E) edge [bend right] (C);
\end{tikzpicture}.\\
case 4.3&case 4.4
\end{tabular}}
\end{table}

\begin{description}
\setlength{\itemsep}{-0.1cm}
\item[Cases 4.1-4.3] If $B \notin S_{CD}$ or $B \notin S_{CE}$ then $B \nb C$ is in $H$ by R1, else $B \bn C$ is in $H$ by R2. Either case is a contradiction.

\item[Case 4.4] Assume that $C \in S_{DE}$. Note that $B \notin S_{DE}$ because, otherwise, R4 would not have been applied. Then, $B \nb C$ is in $H$ by R4, which is a contradiction. On the other hand, assume that $C \notin S_{DE}$. Then, it follows from applying R1 that $H$ has an induced subgraph of the form

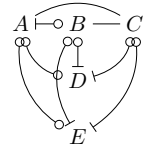
\begin{table}[H]
\centering
\scalebox{0.75}{
\begin{tikzpicture}[inner sep=1mm]
\node at (0,0) (A) {$A$};
\node at (1,0) (B) {$B$};
\node at (2,0) (C) {$C$};
\node at (1,-1) (D) {$D$};
\node at (1,-2) (E) {$E$};
\path[|-o] (A) edge (B);
\path[-] (B) edge (C);
\path[-] (A) edge [bend left] (C);
\path[o-o] (A) edge [bend right] (D);
\path[|-o] (D) edge (B);
\path[o-o] (A) edge [bend right] (E);
\path[|-o] (E) edge [bend left] (B);
\path[|-o] (D) edge [bend right] (C);
\path[|-o] (E) edge [bend right] (C);
\end{tikzpicture}.}
\end{table}

Note that $A \in S_{DE}$ because, otherwise, R4 would not have been applied. Then, $A \bn C$ is in $H$ by R4, which is a contradiction.

\end{description}

\end{description}

\end{proof}

\begin{lemma}\label{lem:oppositeblock}
After line 10, every cycle in $H$ that has an edge $\bn$ also has an edge $\nb$.
\end{lemma}

\begin{proof}
Assume to the contrary that $H$ has a cycle $\rho: V_1, \ldots, V_n=V_1$ that has an edge $\bn$ but no edge $\nb$. Note that every edge in $\rho$ cannot be $\bo$ because, otherwise, every edge in $\rho$ would be $\bb$ by repeated application of R3, which contradicts the assumption that $\rho$ has an edge $\bn$. Therefore, $\rho$ has an edge $-$ or $\nb$. Since the latter contradicts the assumption that the lemma does not hold, $\rho$ has an edge $-$. Assume that $\rho$ is of length three. Then, $\rho$ is of one of the following forms:

\begin{table}[H]
\centering
\scalebox{0.75}{
\begin{tabular}{ccc}
\begin{tikzpicture}[inner sep=1mm]
\node at (0,0) (A) {$V_1$};
\node at (1,0) (B) {$V_{2}$};
\node at (2,0) (C) {$V_{3}$};
\path[|-] (A) edge (B);
\path[-] (B) edge (C);
\path[-] (A) edge [bend left] (C);
\end{tikzpicture}
&
\begin{tikzpicture}[inner sep=1mm]
\node at (0,0) (A) {$V_1$};
\node at (1,0) (B) {$V_{2}$};
\node at (2,0) (C) {$V_{3}$};
\path[|-] (A) edge (B);
\path[-] (B) edge (C);
\path[o-|] (A) edge [bend left] (C);
\end{tikzpicture}
&
\begin{tikzpicture}[inner sep=1mm]
\node at (0,0) (A) {$V_1$};
\node at (1,0) (B) {$V_{2}$};
\node at (2,0) (C) {$V_{3}$};
\path[|-] (A) edge (B);
\path[|-o] (B) edge (C);
\path[-] (A) edge [bend left] (C);
\end{tikzpicture}.
\end{tabular}}
\end{table}

The first form is impossible by Lemma \ref{lem:noundirectedcycle}. The second form is impossible because, otherwise, $V_{2} \nb V_{3}$ would be in $H$ by R3. The third form is impossible because, otherwise, $V_{1} \bn V_{3}$ would be in $H$ by R3. Thus, the lemma holds for cycles of length three.

Assume that $\rho$ is of length greater than three. Recall from above that $\rho$ has an edge $-$ and no edge $\nb$. Let $V_{i+1} - V_{i+2}$ be the first edge $-$ in $\rho$. Assume without loss of generality that $i>0$. Then, $\rho$ has a subpath of the form $V_i \bo V_{i+1} - V_{i+2}$. Note that $V_i \in ad_H(V_{i+2})$ because, otherwise, if $V_{i+1} \notin S_{V_iV_{i+2}}$ then $V_{i+1} \nb V_{i+2}$ would be in $H$ by R1, else $V_{i+1} \bn V_{i+2}$ would be in $H$ by R2. Thus, $H$ has an induced subgraph of one of the following forms:

\begin{table}[H]
\centering
\scalebox{0.75}{
\begin{tabular}{ccc}
\begin{tikzpicture}[inner sep=1mm]
\node at (0,0) (A) {$V_i$};
\node at (1,0) (B) {$V_{i+1}$};
\node at (2,0) (C) {$V_{i+2}$};
\path[|-o] (A) edge (B);
\path[-] (B) edge (C);
\path[-] (A) edge [bend left] (C);
\end{tikzpicture}
&
\begin{tikzpicture}[inner sep=1mm]
\node at (0,0) (A) {$V_i$};
\node at (1,0) (B) {$V_{i+1}$};
\node at (2,0) (C) {$V_{i+2}$};
\path[|-o] (A) edge (B);
\path[-] (B) edge (C);
\path[o-|] (A) edge [bend left] (C);
\end{tikzpicture}
&
\begin{tikzpicture}[inner sep=1mm]
\node at (0,0) (A) {$V_i$};
\node at (1,0) (B) {$V_{i+1}$};
\node at (2,0) (C) {$V_{i+2}$};
\path[|-o] (A) edge (B);
\path[-] (B) edge (C);
\path[|-] (A) edge [bend left] (C);
\end{tikzpicture}.
\end{tabular}}
\end{table}

The first form is impossible by Lemma \ref{lem:noundirectedcycle}. The second form is impossible because, otherwise, $V_{i+1} \nb V_{i+2}$ would be in $H$ by R3. Thus, the third form is the only possible. Note that this implies that $\varrho: V_1, \ldots, V_i, V_{i+2}, \ldots, V_n=V_1$ is a cycle in $H$ that has an edge $\bn$ and no edge $\nb$. 

By repeatedly applying the reasoning above, one can see that $H$ has a cycle of length three that has an edge $\bn$ and no edge $\nb$. As shown above, this is impossible. Thus, the lemma holds for cycles of length greater than three too.
\end{proof}

\begin{theorem}\label{the:acyclic}
After line 11, $H$ is triplex equivalent to $G$ and it has no semidirected cycle.
\end{theorem}

\begin{proof}
Lemma \ref{lem:adjacencies} implies that $G$ and $H$ have the same adjacencies. Lemma \ref{lem:triplexes} implies that $G$ and $H$ have the same triplexes. Lemma \ref{lem:oppositeblock} implies that $H$ has no semidirected cycle.
\end{proof}

\section{Discussion}\label{sec:discussion}

In this paper, we have presented an algorithm for learning an AMP CG a given probability distribution $p$ is faithful to. In practice, of course, we do not usually have access to $p$ but to a finite sample from it. Our algorithm can easily be modified to deal with this situation: Replace $A \ci_p B | S$ in line 6 with a hypothesis test, preferably with one that is consistent so that the resulting algorithm is asymptotically correct.

It is worth mentioning that, whereas R1, R2 and R4 only involve three or four nodes, R3 may involve many more. Hence, it would be desirable to replace R3 with a simpler rule such as 

\begin{table}[H]
\centering
\scalebox{0.75}{
\begin{tabular}{ccc}
\begin{tabular}{c}
\begin{tikzpicture}[inner sep=1mm]
\node at (0,0) (A) {$A$};
\node at (1,0) (B) {$B$};
\node at (2,0) (C) {$C$};
\path[|-o] (A) edge (B);
\path[|-o] (B) edge (C);
\path[o-o] (A) edge [bend left] (C);
\end{tikzpicture}
\end{tabular}
& $\Rightarrow$ &
\begin{tabular}{c}
\begin{tikzpicture}[inner sep=1mm]
\node at (0,0) (A) {$A$};
\node at (1,0) (B) {$B$};
\node at (2,0) (C) {$C$};
\path[|-o] (A) edge (B);
\path[|-o] (B) edge (C);
\path[|-o] (A) edge [bend left] (C);
\end{tikzpicture}.
\end{tabular}
\end{tabular}}
\end{table}

Unfortunately, we have not succeeded so far in proving the correctness of our algorithm with such a simpler rule. Note that the output of our algorithm will be the same whether we keep R3 or we replace it with a simpler sound rule. The only benefit of the simpler rule may be a decrease in running time.

We have shown in Lemma \ref{lem:triplexes} that, after line 11, $H$ has all the immoralities in $G$ or, in other words, every flag in $H$ is in $G$. The following lemma strengthens this fact.

\begin{lemma}\label{lem:essentialline}
After line 11, every flag in $H$ is in every CG $F$ that is triplex equivalent to $G$.
\end{lemma}

\begin{proof}
Note that every flag in $H$ is due to an induced subgraph of the form $A \bn B \bb C$. Note also that all the blocks in $H$ follow from the adjacencies and triplexes in $G$ by repeated application of R1-R4. Since $G$ and $F$ have the same adjacencies and triplexes, all the blocks in $H$ hold in both $G$ and $F$ by Lemma \ref{lem:soundness}. 
\end{proof}

The lemma above implies that, in terms of \cite{RoveratoandStudeny2006}, our algorithm outputs a deflagged graph. \cite{RoveratoandStudeny2006} also introduce the concept of strongly equivalent CGs: Two CGs are strongly equivalent iff they have the same adjacencies, immoralities and flags. Unfortunately, not every edge $\ra$ in $H$ after line 11 is in every deflagged graph that is triplex equivalent to $G$, as the following example illustrates, where both $G$ and $H$ are deflagged graphs. 

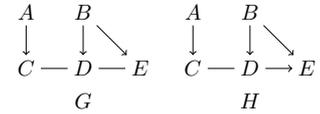
\begin{table}[H]
\centering
\scalebox{0.75}{
\begin{tabular}{cc}
\begin{tikzpicture}[inner sep=1mm]
\node at (0,0) (A) {$A$};
\node at (1,0) (B) {$B$};
\node at (0,-1) (C) {$C$};
\node at (1,-1) (D) {$D$};
\node at (2,-1) (E) {$E$};
\path[->] (A) edge (C);
\path[->] (B) edge (D);
\path[-] (C) edge (D);
\path[-] (D) edge (E);
\path[->] (B) edge (E);
\end{tikzpicture}
&
\begin{tikzpicture}[inner sep=1mm]
\node at (0,0) (A) {$A$};
\node at (1,0) (B) {$B$};
\node at (0,-1) (C) {$C$};
\node at (1,-1) (D) {$D$};
\node at (2,-1) (E) {$E$};
\path[->] (A) edge (C);
\path[->] (B) edge (D);
\path[-] (C) edge (D);
\path[->] (D) edge (E);
\path[->] (B) edge (E);
\end{tikzpicture}\\
$G$&$H$
\end{tabular}}
\end{table}

Therefore, in terms of \cite{RoveratoandStudeny2006}, our algorithm outputs a deflagged graph but not the largest deflagged graph. The latter is a distinguished member of a class of triplex equivalent CGs. Fortunately, the largest deflagged graph can easily be obtained from any deflagged graph in the class \cite[Corollary 17]{RoveratoandStudeny2006}. 

The correctness of our algorithm lies upon the assumption that $p$ is faithful to some CG. This is a strong requirement that we would like to weaken, e.g. by replacing it with the milder assumption that $p$ satisfies the composition property. Correct algorithms for learning directed and acyclic graphs (a.k.a Bayesian networks) under the composition property assumption exist \citep{ChickeringandMeek2002,Nielsenetal.2003}. We have recently developed a correct algorithm for learning LWF CGs under the composition property \citep{Pennaetal.2012}. The way in which these algorithms proceed (a.k.a. score+search based approach) is rather different from that of the algorithm presented in this paper (a.k.a. constraint based approach). In a nutshell, they can be seen as consisting of two phases: A first phase that starts from the empty graph $H$ and adds single edges to it until $p$ is Markovian wrt $H$, and a second phase that removes single edges from $H$ until $p$ is Markovian wrt $H$ and $p$ is not Markovian wrt any CG $F$ st $I(H) \subseteq I(F)$. The success of the first phase is guaranteed by the composition property assumption, whereas the success of the second phase is guaranteed by the so-called Meek's conjecture \citep{Meek1997}. Specifically, given two directed and acyclic graphs $F$ and $H$ st $I(H) \subseteq I(F)$, Meek's conjecture states that we can transform $F$ into $H$ by a sequence of operations st, after each operation, $F$ is a directed and acyclic graph and $I(H) \subseteq I(F)$. The operations consist in adding a single edge to $F$, or replacing $F$ with a triplex equivalent directed and acyclic graph. Meek's conjecture was proven to be true in \citep[Theorem 4]{Chickering2002}. The extension of Meek's conjecture to LWF CGs was proven to be true in \citep[Theorem 1]{Penna2011}. Unfortunately, the extension of Meek's conjecture to AMP CGs does not hold, as the following example illustrates.

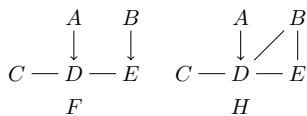
\begin{table}[H]
\centering
\scalebox{0.75}{
\begin{tabular}{cc}
\begin{tikzpicture}[inner sep=1mm]
\node at (0,0) (A) {$A$};
\node at (1,0) (B) {$B$};
\node at (-1,-1) (C) {$C$};
\node at (0,-1) (D) {$D$};
\node at (1,-1) (E) {$E$};
\path[->] (A) edge (D);
\path[->] (B) edge (E);
\path[-] (C) edge (D);
\path[-] (D) edge (E);
\end{tikzpicture}
&
\begin{tikzpicture}[inner sep=1mm]
\node at (0,0) (A) {$A$};
\node at (1,0) (B) {$B$};
\node at (-1,-1) (C) {$C$};
\node at (0,-1) (D) {$D$};
\node at (1,-1) (E) {$E$};
\path[->] (A) edge (D);
\path[-] (B) edge (E);
\path[-] (C) edge (D);
\path[-] (D) edge (E);
\path[-] (B) edge (D);
\end{tikzpicture}\\
$F$&$H$
\end{tabular}}
\end{table}

Then, $I(H)=\{X \ci_H Y | Z : X \ci_H Y | Z \in I_1(H) \cup I_2(H) \lor Y \ci_H X | Z \in I_1(H) \cup I_2(H) \}$ where $I_1(H)=\{A \ci_H Y | Z : Y, Z \subseteq B \cup C \cup E \land D \notin Z \}$ and $I_2(H)=\{C \ci_H Y | Z : Y \subseteq B \cup E \land A \cup D \subseteq Z \}$. One can easily confirm that $I(H) \subseteq I(F)$ by using the definition of separation. However, there is no CG that is triplex equivalent to $F$ or $H$ and, obviously, one cannot transform $F$ into $H$ by adding a single edge.

While the example above compromises the development of score+search learning algorithms that are correct and efficient under the composition property assumption, it is not clear to us whether it also does it for constraint based algorithms. This is something we plan to study.

\section*{Acknowledgments}
\small
This work is funded by the Center for Industrial Information Technology (CENIIT) and a so-called career contract at Link\"oping University, and by the Swedish Research Council (ref. 2010-4808).

\end{document}